%% file: main.tex
\documentclass[10pt,twocolumn,letterpaper]{article}
\usepackage{cvpr}              
\usepackage{float}
\input{preamble}

\input{math_commands.tex}
\usepackage{amsthm}
\definecolor{mylightgray}{gray}{0.94}
\definecolor{cvprblue}{rgb}{0.21,0.49,0.74}
\usepackage[pagebackref,breaklinks,colorlinks,citecolor=cvprblue]{hyperref}
\usepackage{multicol}
\usepackage{multirow, boldline}
\usepackage{colortbl}
\usepackage{caption}

\usepackage{mdframed}
\usepackage{xcolor}

\newmdenv[
  leftline=false,
  rightline=false,
  topline=false,
  bottomline=false,
  backgroundcolor=gray!10,
  innerleftmargin=10pt,
  innerrightmargin=10pt,
  innertopmargin=10pt,
  innerbottommargin=10pt
]{custombox}

\definecolor{Gray}{gray}{0.9}
\newtheorem{theorem}{Theorem}

\newtheorem{definition}{Definition}
\usepackage{hyperref}
\def\paperID{3740} 
\def\confName{CVPR}
\def\confYear{2024}

\title{Diffusion Time-step Curriculum for One Image to 3D Generation}

\author{
\textbf{Xuanyu Yi}\textsuperscript{1,4}, \textbf{Zike Wu}\textsuperscript{1}, \textbf{Qingshan Xu}\textsuperscript{1}, \textbf{Pan Zhou}\textsuperscript{3*},\\
\textbf{Joo-Hwee Lim}\textsuperscript{4}, \textbf{Hanwang Zhang}\textsuperscript{2,1}\\
{\small \textsuperscript{1}Nanyang Technological University, \textsuperscript{2}Skywork AI} \\
{\small \textsuperscript{3}Singapore Management University, \textsuperscript{4}Institute for Infocomm Research}\\
{\tt\small xuanyu001@e.ntu.edu.sg, zike001@e.ntu.edu.sg, qingshan.xu@ntu.edu.sg,}\\
{\tt\small panzhou@smu.edu.sg, joohwee@i2r.a-star.edu.sg, hanwangzhang@ntu.edu.sg}
}

\begin{document}

\twocolumn[{
\renewcommand\twocolumn[1][]{#1}
\maketitle
\centering
\vspace{-0.5cm}
 \includegraphics[width=\linewidth]{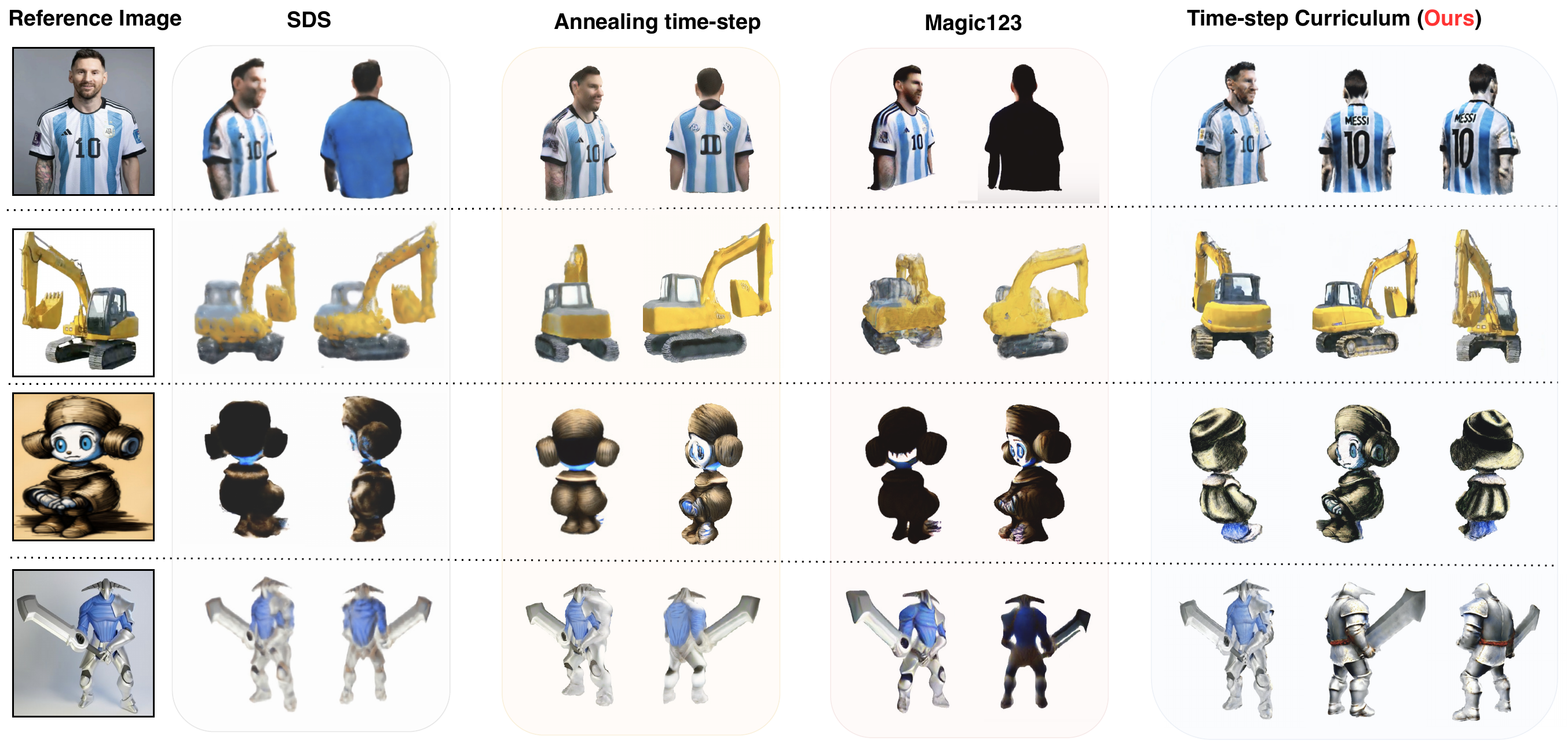}
 
 \captionsetup{type=figure}
\caption{DTC123 can generate high-fidelity and multiview-consistent 3D assets from a single arbitrary image, significantly enhancing the original SDS pipeline through the \textit{diffusion time-step curriculum}.}
\vspace{0.2cm}
\label{fig:1}
}]

\maketitle

{\let\thefootnote\relax\footnotetext{$^*$Corresponding author.}}

\input{sec/0_abstract}
\input{sec/1_intro}
\input{sec/2_related}
\input{sec/3_preliminary}

\input{sec/4_methods}
\input{sec/6_experiment}

\input{sec/7_discussion}

\input{sec/8_conclusion}

\clearpage
{
    \small
    \bibliographystyle{ieeenat_fullname}
    \bibliography{main}
}

\input{sec/X_suppl}

\end{document}

%% file: preamble.tex
\usepackage[dvipsnames]{xcolor}


%% file: math_commands.tex
\usepackage{amsmath,amsfonts,bm}

\def\1{\bm{1}}

\def\rmI{{\mathbf{I}}}

\DeclareMathAlphabet{\mathsfit}{\encodingdefault}{\sfdefault}{m}{sl}
\SetMathAlphabet{\mathsfit}{bold}{\encodingdefault}{\sfdefault}{bx}{n}

\def\gN{{\mathcal{N}}}

\newcommand{\KL}{D_{\mathrm{KL}}}

\newcommand{\bbx}{\mathbf{x}}

%% file: sec/0_abstract.tex
\begin{abstract}

Score distillation sampling~(SDS) has been widely adopted to overcome the absence of unseen views in reconstructing 3D objects from a \textbf{single} image. It leverages pre-trained 2D diffusion models as teacher to guide the reconstruction of student 3D models. Despite their remarkable success, SDS-based methods often encounter geometric artifacts and texture saturation. We find out the crux is the overlooked indiscriminate treatment of diffusion time-steps during optimization: it unreasonably treats the student-teacher knowledge distillation to be equal at all time-steps and thus entangles coarse-grained and fine-grained modeling. Therefore, we propose the Diffusion Time-step Curriculum one-image-to-3D pipeline (DTC123), which involves both the teacher and student models collaborating with the time-step curriculum in a coarse-to-fine manner. Extensive experiments on NeRF4, RealFusion15, GSO and Level50 benchmark demonstrate that DTC123 can produce multi-view consistent, high-quality, and diverse 3D assets. Codes and more generation demos will be released in \url{https://github.com/yxymessi/DTC123}.
\end{abstract}

%% file: sec/1_intro.tex

\section{Introduction}\label{sec:intro}

We consider the problem of obtaining a 3D asset from a \textit{single} image. This endeavor holds tremendous industrial promise, notably in realms such as AR/VR content creation from a single snapshot and enhancing robotic navigation through individual captures~\cite{sun2023trosd,gul2019comprehensive}. However, reconstructing 3D models (\eg, NeRF~\cite{mildenhall2021nerf,barron2021mip,muller2022instant} and mesh~\cite{kato2018neural,liu2019soft}) from a single image has remained a formidable challenge due to its severely ill-posed nature, as one image does not contain sufficient unseen views of a 3D scene.

\begin{figure*}[t]
    \centering
    \includegraphics[width=\linewidth]{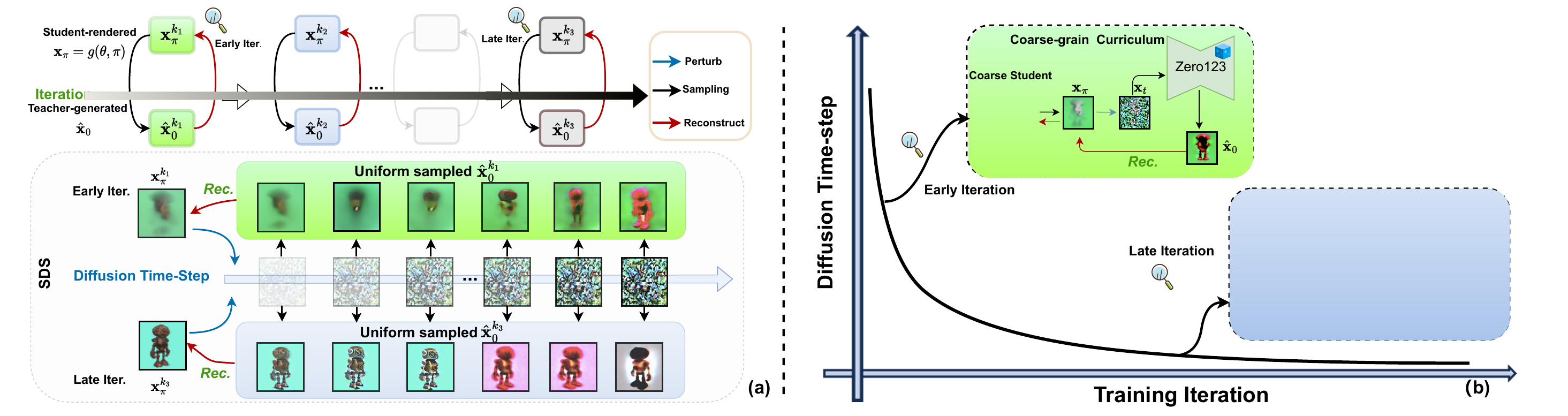}
    \caption{\textbf{(a)} SDS embraces an symbiotic teacher-student cycle with the training iteration progresses (\textit{Top}). However, it entangles coarse-grained and fine-grained modeling with uniform sampling of time steps (\textit{Bottom}) and equal treatment of student and teacher, where $k_1 \ldots k_3 $ denotes the training iteration from early to late. \textbf{(b)} Our DTC123 follows the \textit{diffusion time-step curriculum}, where larger time steps capture coarse-grained concept and smaller time steps focus on fine-grained details.}
    \label{fig:2}
\end{figure*}

Fortunately, recent advances in large-scale pretrained 2D diffusion models~\cite{zhang2023text,rombach2022high,saharia2022photorealistic} have paved the way for such an ill-posed challenge by synthesizing the unseen views of quality. Pioneered by Score Distillation Sampling (SDS)~\cite{poole2022dreamfusion}, we can use the 2D models as \textit{teacher} to guide the reconstruction of the 3D models as \textit{student}.
The key is rooted in a symbiotic teacher-student cycle: As illustrated in Figure~\ref{fig:2} (\textit{Top}), in the early iterations with significantly flawed student output, the teacher provides the rough shape of objects as guidance;
as the student gradually improves, it reciprocates with more precise conditions for the teacher, who in turn provides more accurate and fine-grained supervision.
Along with the teacher's diffusion time steps, SDS optimizes the student by minimizing the reconstruction error between the student-rendered and the teacher-generated 2D images (see details in Sec.~\ref{sds_rec}).

However, 3D models generated by SDS have defects that cannot be ignored. As shown in Figure~\ref{fig:1}, empirical observations indicate that they often encounter collapsed geometry and limited fidelity.
Such issues arise primarily from the confusion of holistic structures and local details as stated in the field of novel view synthesis~\cite{gao2022nerf,riegler2021stable,wiles2020synsin}.
We conjecture that this may be attributed to the uniform sampling of diffusion time steps during the calculation of reconstruction errors, that is, SDS treats the student-teacher knowledge distillation to be \emph{equal at all time steps}. This is counter-intuitive, \eg, the teacher shouldn't teach low-level details when the student is still grappling with high-level concepts.

Our key insight is that an optimal SDS should follow a \textit{diffusion time-step curriculum} : larger time steps capture coarse-grained knowledge like geometry formation and smaller time steps focus on enhancing fine-grained details like texture nuance. To this end, we propose the \textbf{D}iffusion \textbf{T}ime-step \textbf{C}urriculum \textbf{1} image \textbf{to}  \textbf{3}D pipeline dubbed \textbf{DTC123}, where both the teacher and the student model collaborate with the annealing time-step, exhibiting a coarse-to-fine generation process.
More concretely,
\begin{itemize}



\item  \textbf{Student-wise}: 3D models should progress from low-resolution concepts to high-resolution. We leverage resolution constraints from the hash-encoding band (NeRF) and the tet grid (DMTet) to gradually absorb knowledge, beginning with broader structural elements and eventually focusing on localized textures and complex scene illumination (Sec.~\ref{student}).

\item  \textbf{Teacher-wise}: diffusion models should prioritize on coarse shape to visual details. We employ a pose-aware prior, Zero-1-to-3 ~\cite{liu2023zero}, to establish a coarse-grained structure that aligns with the reference image. Subsequently, we harness the combined guidance of Zero-1-to-3  and Stable Diffusion, with LLM-augmented prompt and multi-step sampling (Sec.~\ref{teacher}), to further provide fine-grained texture intricacies.

\end{itemize}

In addition, we introduce several geometric regularization to alleviate the Janus Face~\cite{hong2023debiasing, metzer2023latent} and high-frequency surface artifacts (Sec.~\ref{reg}).
As illustrated in Figure~\ref{fig:1}, by integrating the time-step curriculum, which includes three coherent parts --- time-step schedule, progressive student representation and teacher guidance with the aforementioned regularization techniques, DTC123 significantly enhances the geometry quality and texture fidelity of the SDS-based pipeline. 
We demonstrate the superiority of DTC123 on NeRF4~\cite{mildenhall2021nerf}, RealFusion15~\cite{melas2023realfusion}, GSO~\cite{downs2022google} and our benchmark Level50. Comprehensive quantitative and qualitative evaluations clearly show that DTC123 can efficiently generate multi-view consistent, high-fidelity, and diverse 3D assets, continually outperforming other state-of-the-art methods. Furthermore, DTC123 enables wide range of applications, \eg, multi-instance generation and mesh refinement.
In summary, we make three-fold contributions:


\begin{itemize}
\setlength\itemsep{0em}
\item We develop an end-end one image-3D pipeline DTC123, boosting the efficiency, diversity and fidelity of SDS-based methods in real-world and synthetic scenarios.
\item \textit{Diffusion time-step curriculum} is a plug-and-play training principle that could further unleash the potential of SDS-based teacher-student models.
\item We systematically and theoretically validate the proposed diffusion time-step curriculum.

\end{itemize}

%% file: sec/2_related.tex
\section{Related Work} \label{sec:related_work}

\noindent \textbf{Text-3D Generation} focuses on generating 3D assets from a given text prompt. The core mechanism of such approaches is the score distillation sampling (SDS) proposed by ~\cite{poole2022dreamfusion}, where the diffusion priors are used as score functions to supervise the optimization of a 3D representation. Recent advancements aim to enhance the training stability and generation fidelity via advanced shape guidance~\cite{shi2023mvdream,zhao2023efficientdreamer,li2023sweetdreamer}, disentangled 3D representation~\cite{wu2023hd,chen2023fantasia3d, xu2023matlaber} and loss design~\cite{wang2023prolificdreamer,zhu2023hifa,katzir2023noise,wu2024consistent3d}. Note that some concurrent works~\cite{zhu2023hifa,huang2023dreamtime,tang2023dreamgaussian} also leverage annealed time-step schedule for efficient training, but they fail to combine such time-step sampling schedule with the teacher-student knowledge transfer, \ie, thus have not fully unleashed the potential of their symbiotic cycle.


\noindent \textbf{Image-3D Generation} focuses on generating 3D assets from a given reference image, which can also be considered as an ill-posed single-view reconstruction problem~\cite{deng2023nerdi}. The above text-to-3D methods can be directly adapted for image-to-3D generation with image captioning, \eg, Realfusion~\cite{melas2023realfusion} and NeRDi~\cite{deng2023nerdi} directly adapted SDS~\cite{poole2022dreamfusion} into Image-to-3D with textual inversion of the given image. Magic123~\cite{qian2023magic123} and Consistent123~\cite{weng2023consistent123} combine Zero-1-to-3~\cite{liu2023zero} with RealFusion~\cite{melas2023realfusion} to further improve the generation quality. One-2-3-45~\cite{liu2023zero} and Dreamgaussian~\cite{tang2023dreamgaussian} adopted NeuS~\cite{wang2021neus,wang2023neus2} and Generative Gaussian Splatting~\cite{kerbl20233d} as the 3D representation respectively, significantly reducing the generation time at the expense of the generation quality. We build upon these optimization-based approaches by implementing a coarse-to-fine optimization strategy that explores the \textit{diffusion time-step curriculum} they overlooked, which enables efficient generation of high-fidelity, multi-view consistent 3D assets from one image.

%% file: sec/3_preliminary.tex
\section{Preliminary} \label{preliminary} 
Given a single image of an object, our goal is to optimize a coherent 3D model (\eg, NeRF, mesh) so that it can restore the given image from the reference view $\pi_{\text{ref}}$ and generate a highly plausible image from any unseen view $\pi$ with the supervision of the diffusion model.
Here, we mainly introduce Score Distillation Sampling (SDS) ~\cite{poole2022dreamfusion} with \textit{student} 3D models, \textit{teacher} diffusion models (see \textit{Appendix}) that will help to build up our approach.
\label{diffusion}

\subsection{Score Distillation Sampling (SDS)}
\label{sds_rec}
SDS~\cite{poole2022dreamfusion} distills 2D priors from a pre-trained conditional diffusion \textit{teacher} model $\boldsymbol\epsilon_{\phi}(\cdot)$ into differentiable 3D \textit{student} representations $\theta$.
In particular, given a certain camera parameter $\pi$, we \textbf{randomly} select a diffusion time-step $t$ and perturb the student-rendered $\bbx_{\pi} \xrightarrow{+\sigma_{t} \boldsymbol{\epsilon}} \bbx_{t}$ by adding a Gaussian noise $\boldsymbol{\epsilon}$, and reformulate SDS from the perspective of reconstruction by calculating:
\vspace{-2mm}
\begin{align}
\label{sds}
\nabla_{ \boldsymbol\theta} \mathcal{L}_{\text{SDS}}( \boldsymbol \theta, t) = &
\mathbb{E}_{\boldsymbol{\epsilon}}\left[\omega(t)\left(\boldsymbol{\epsilon}_{\phi}\left(\bbx_{t}; t, \boldsymbol{y}\right)-\boldsymbol{\epsilon}\right) \frac{\partial \bbx}{\partial  \boldsymbol \theta}\right] \notag \\
= & \mathbb{E}_{\boldsymbol{\epsilon}}\left[ \bar{{\omega}}(t) \left(\bbx_{\pi}-\hat{\bbx}_{0}\right) \frac{\partial \bbx}{\partial  \boldsymbol \theta} \right],
\end{align}
where $ \hat{\bbx}_{0} = \bbx_{t}-  \sigma_t \boldsymbol{\epsilon}_{\phi}\left(\bbx_{t} ; \boldsymbol{y}, t\right)$, can be considered as single-step de-nosing output with starting point $\bbx_{t}$; $\boldsymbol{y}$ is the condition (\eg, text, camera pose) and depends on the types of teacher diffusion models; $\bar{{\omega}}(t) =\omega(t) / \sigma_t$, denotes the weight function. Thus, we reveal that the crux of this teacher-student optimization process is directly determined by such perturbed-and-denoised output, \ie, the quality of the teacher-generated $\hat\bbx_{0}$. As illustrated in Figure~\ref{fig:2}(a), intuitively, not all time-steps $t$ can provide useful and valid guidance $\hat\bbx_{0}$, which motivates the following proposed diffusion time-step curriculum.


%% file: sec/4_methods.tex
\section{DTC123}
\label{methods}

In this section, we introduce our \textit{diffusion time-step curriculum} one-image-to-3D pipeline, called ``DTC123", with a theoretical justification in \textit{Appendix}. As shown in Figure~\ref{pipeline}(a): We start by taking a reference image and extracting its geometric estimation and text description. Then the optimization procedure can be categorized into unseen view guidance and reference restoration: 


\begin{itemize}  
    \item For the \textbf{unseen view}, we employ our \textit{diffusion time-step curriculum}, where larger time-steps capture coarse-grained concepts, and smaller time-steps focus on fine-grained details. Specifically, we implement such a curriculum with an annealed time-step sampling schedule (Sec.~\ref{timestep}), progressive \textit{student} representation (Sec.~\ref{student}), and coarse-to-fine \textit{teacher} guidance (Sec.~\ref{teacher}).
    \item For the \textbf{reference view}, we basically penalize the 3D model to align with the given image, employing traditional reconstruction constraints.
\end{itemize} 
Moreover, DTC123 incorporates several techniques (Sec.~\ref{reg}) to enhance generation efficiency, geometric robustness, and alleviate the Janus Face problem.

\begin{figure*}[t]
  \begin{minipage}[b]{1.0\linewidth}
  \centerline{\includegraphics[width=0.98\linewidth]{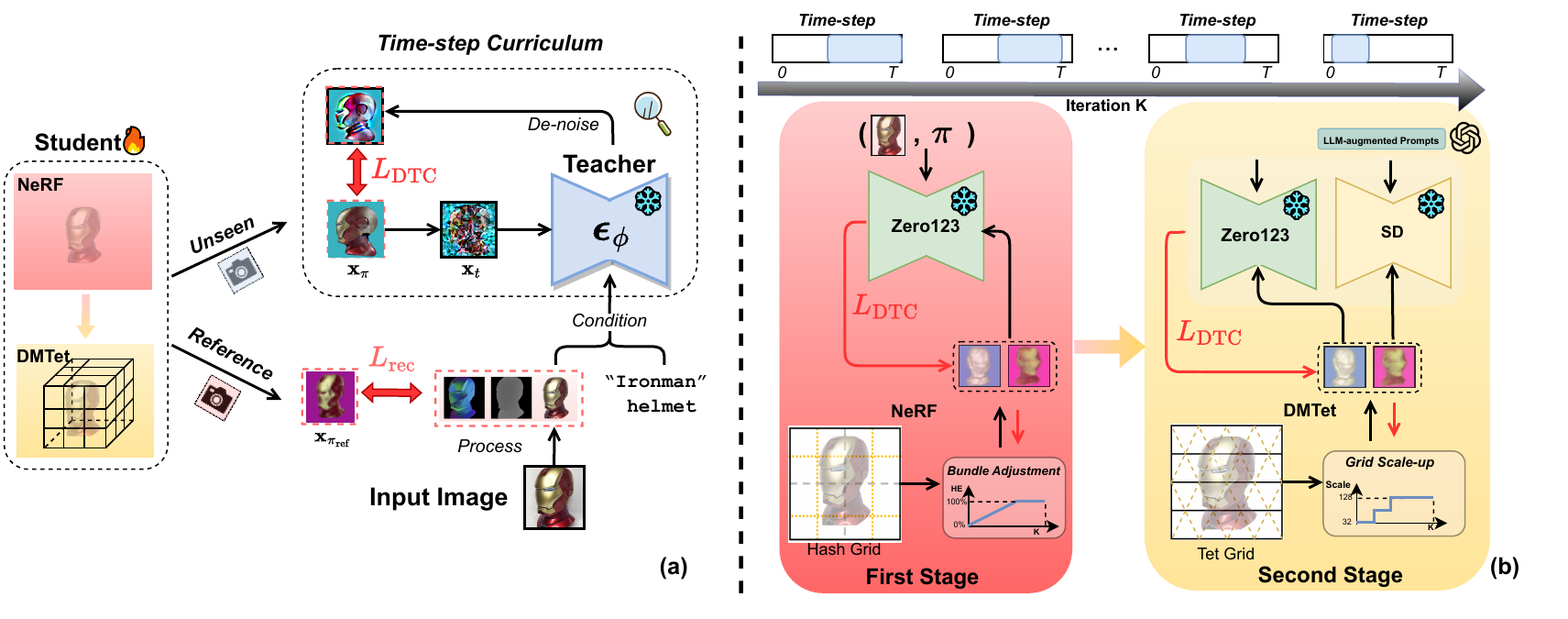}}
  \end{minipage}
  \caption{\textbf{(a)} Overall pipeline of DTC123, which have two optimization stages and includes the reference view reconstruction and unseen view imagination. \textbf{(b)} The zoom-in diagram of unseen view imagination with the proposed \textit{diffusion time-step curriculum}.}
  \label{pipeline}
\end{figure*}

\subsection{Annealed Time-Step Schedule}
\label{timestep}


We develop a time-step annealed sampling schedule for our time-step curriculum, thereby \textit\textbf{facilitating} teacher-student knowledge transfer in a progressive, coarse-to-fine manner. Please refers to \textit{Appendix} for its theoretical justification.
In particular, we use an annealed interval \([t_{\text{mid}} - \Delta, t_{\text{mid}} + \Delta]\) to randomly select the time-step \( t \), with the interval midpoint \( t_{\text{mid}} \) decreasing monotonically.
The interval radius \( \Delta \) narrows progressively during optimization, and the midpoint \( t_{\text{mid}} \) is determined by the current iteration $k$:
\begin{equation}
\label{eq:anneal}
t_{\text{mid}} := t_{\max} - (t_{\max} - t_{\min}) \cdot \log\left(1 + \frac{\lfloor k / l \rfloor \cdot l}{N}\right),
\end{equation}
where \( t_{\max} \) and \( t_{\min} \) refers to the entire diffusion time-step range; \( l \) and \( N \) denotes the step length and total training iteration, respectively. 
This schedule not only introduces \textit{local randomness} to preserve the vibrancy of the model's coloration but also allocates more iterations to smaller time-steps for in-depth detail exploration. 







\subsection{Student: Progressive 3D Representation}
\label{student}

The student 3D model adheres to the \textit{diffusion time-step curriculum} by initially representing coarse-grained features at larger time-steps and subsequently fine-grained details at smaller time-steps. Due to computational memory constraints~\cite{lin2023magic3d,qian2023magic123}, we leverage NeRF~\cite{muller2022instant} for low-resolution scene modeling in the \textit{first} stage, and then adopt DMTet~\cite{shen2021deep} for high-resolution mesh fine-tuning in the \textit{second} stage.
Recall that grid-based 3D models embrace the inherent multi-resolution representation~\cite{yu2022monosdf,sun2022direct,wang2023neus2}, where the lower-level spatial grids (\eg, hash grid, tetrahedral grid) store the general contours while the higher-level counterparts store the finer textures and scene illumination. Accordingly, we employ progressive resolution constraints in \textbf{both} stages, which allows the student to initially assimilate the correct geometry structure information with a smooth coarse-grained signal and later shifts to learn a high-fidelity scene representation. More concretely: 

\noindent \textbf{Progressive bundle for NeRF.}
The design of hash encoding captures both low and high-resolution features, which might inadequately represent basic geometry at large time-step as inaccurate and diverse finer details could overshadow essential structural information. To mitigate this issue, given the hash-grid encoding (HE), we gradually regularize the visible resolution bands by applying a dynamic soft feature mask: 


\begin{equation}
\begin{split}
    m_i(k, N, L) = \left \{
\begin{array}{ll}
    1, & \text{if}~i \leq 4 + \min\left(\left\lfloor \frac{10k}{N} \right\rfloor, L-4\right) \\
    0, & \text{otherwise}
\end{array}
\right.
\end{split}
\end{equation}
where $i$ denotes the HE feature level from coarsest to finest and $L$ denotes the total number of feature levels. 






\noindent \textbf{Progressive Tetrahedral Grid for DMTet.}
As shown in Figure~\ref{pipeline}(b), we initially convert~\cite{stable-dreamfusion} the neural density field to a signed distance function (SDF) field and re-use the  neural color field for texture representation, progressing to higher resolutions \eg, $64 \times 64 \rightarrow 512 \times 512$ for detailed surface rendering. To circumvent issues like mesh distortion or topology errors from abrupt resolution changes, we employ a graded approach, scaling up the tetrahedral grid size, \ie $32 \rightarrow 64 \rightarrow 128$, and the rasterizing resolution progressively with the reducing time-step. This strategy ensures a seamless transition from initialized structural capture to rendering sophisticated surface nuances.

\subsection{Teacher: Coarse-to-fine Diffusion Prior }
\label{teacher}

The teacher diffusion model should follow the \textit{diffusion time-step curriculum} by initially offering a rough silhouette of the object in the desired pose in larger time steps, subsequently prioritizing intricate texture refinement in small time steps. A natural question follows : \textit{What is the suitable teacher for the time-step curriculum?} Empirical evidence\footnote{To answer this question, we quantitatively analyze two most popular teacher diffusion models by investigating the contour exploration consistency via  MaskIoU and the perceptual generation quality computed by CLIP-R~\cite{liu2023one} in the OmniObject3D~\cite{wu2023omniobject3d} dataset. See detailed experimental results in \textit{Appendix}.} suggests: the view-conditioned Zero-1-to-3 serves as a coarse-grained teacher by providing a more accurate geometry structure at large time-step. Conversely, the Stable Diffusion is suitable for a fine-grained teacher as it yields realistic texture details at smaller time-step.

Such evaluation drives our search for the following utilization of teacher guidance that makes the best use of both diffusion priors~\cite{yi2023invariant}.
As illustrated in Figure~\ref{pipeline}(b), in the first stage, we use Zero-1-to-3 solely for efficient coarse-grained contour and structural guidance. Subsequently, in the second stage, a collaborative diffusion guidance with dynamic prior reweighting is leveraged, wherein Zero-1-to-3 guides SDF field geometry optimization supervised by textureless shading; Stable Diffusion aids in the optimization of color field respectively at smaller time-steps. For clarity, we denote the Zero-1-to-3 guidance as $\mathcal{L}_{\text{SDS}}^{\text{geo}}( \boldsymbol \theta, t)$ and the Stable Diffusion guidance as $\mathcal{L}_{\text{SDS}}^{\text{tex}}( \boldsymbol \theta, t)$~\footnote{Note that the weight function $\bar{\omega}(t)$ is calibrated relative to the spherical distance from the reference view, as 3D models tend to reconstruct more accurately near the reference view, requiring less teacher-guided imagination.}. The diffusion objective for the time-step curriculum can then be formulated as:
\begin{equation}
\nabla_{\theta} \mathcal{L}_{\text{DTC}}(\theta,t) =  \nabla_{\theta} \mathcal{L}_{\text{SDS}}^{\text{geo}}( \boldsymbol \theta, t) + \nabla_{\theta} \lambda \mathcal{L}_{\text{SDS}}^{\text{tex}}( \boldsymbol \theta, t),
\end{equation}
where the time-step $t$ follows the sampling schedule in Sec.~\ref{timestep}; $\lambda$ is the trade-off hyper-parameter, set as 0 in the first stage and gradually increase with the reducing time-step in the second stage.
We apply DDIM sampling~\cite{song2020denoising,gabbur2023improved} for multi-step de-nosing process, \ie, $\bbx_{t} \rightarrow \bbx_{t-r} \ldots \rightarrow \hat{\bbx}_{0}$ instead of the imprecise single-step de-noising in the last few iterations for both efficiency and performance concerns. Empirically, it equips the teacher-generated $\hat{\bbx}_{0}$  with more decent details and mitigates issues like texture flickering or overly saturated color blocks at smaller time-steps.
In addition, we propose the following approaches to further improve the fine-grained supervision quality in scenarios with sophisticated textures by clarifying the ambiguous inferred text description and alleviating the Janus face problem.

\noindent \textbf{LLM-augmented Prompts.}
Previous methods~\cite{metzer2023latent, lin2023magic3d, tang2023make} directly utilize ambiguous view-dependent prompts within $\{$\texttt{front}, \texttt{back}, \texttt{side}$\}$ as the condition for Stable Diffusion. Paralleled to ~\cite{wang2023score}, we notice that more specific language prompts force the image distribution of Stable Diffusion to be narrower and more beneficial to the mode-seeking SDS algorithm. To this end, we instead leverage instruction-tuned large language models (LLMs)~\cite{ray2023chatgpt} with carefully designed task prompts to specify unseen view descriptions from the original text description. The core objective of LLM-augmented prompts is to intricately enhance multi-view descriptions by integrating additional details, \eg, \textit{`\textbf{The back of} Ironman helmet, \textbf{with metallic sheen'}}, while meticulously avoiding description conflict among different views. Detailed instruct-LLM design is in \textit{Appendix}.


\noindent \textbf{Camera Pose De-biasing.}
When the given object is human / animal-like, the intrinsic pre-trained bias~\cite{liu2023zero,hong2023debiasing,armandpour2023re} of diffusion models always causes the Janus face problems. Empirical observations suggest that Zero-1-to-3 tends to generate a Janus face on the backside, tends to 'copy' the conditioned front face to the back view with symmetric contour, whereas Stable Diffusion is more prone to this issue on the sides. To mitigate these anomalies in the second stage, we dynamically employ gradient clipping and randomized dropout for $\nabla_{\theta} \mathcal{L}_{\text{SDS}}^{c}( \boldsymbol \theta, t)$ within the azimuth range $\left[\frac{11\pi}{12}, \frac{13\pi}{12}\right]$ and $\nabla_{\theta}  \mathcal{L}_{\text{SDS}}^{f}( \boldsymbol \theta, t)$ within the azimuth range $\pm\left[\frac{\pi}{6}, \frac{\pi}{4}\right]$, respectively.





\subsection{Advanced Regularization Techniques} 
\label{reg}

Due to the ambiguity and inconsistency guidance of diffusion-based 3D generation~\cite{tang2023dreamgaussian,graikos2022diffusion}, high-frequency artifacts often appear on the crisped surface of the student renderings. To counteract this, following~\cite{stable-dreamfusion,zhu2023rhino}, we basically improve the smoothness of the normal map by calculating the normal vector using finite depth differences and incorporate additional continuous connection from the input coordinate for hash-based NeRF regularization. As for DMTet and  mesh exportation, we apply Laplacian-based regularization~\cite{yin2015laplacian} to achieve mesh smoothness, capitalizing on the uniform Laplacian matrix derived from the mesh vertices and adjacent faces, as well as removal and calibration of unreferenced vertices and faces. We denote the above smoothness regularization as $\mathcal{L}_{\text {reg }}$ in both stages.

 




\subsection{Training Objective}
\label{objective}



Our final objective then comprises three key terms:
\begin{equation}
\nabla_{\theta} \mathcal{L}=\nabla \mathcal{L}_{\text{DTC}}+\lambda_{\text {reg }} \cdot \nabla \mathcal{L}_{\text {reg }} 
+\lambda_{\text {rec }} \cdot \nabla \mathcal{L}_{\text {rec }},
\end{equation}
where $\mathcal{L}_{\text{DTC}}$ represents the time-step curriculum diffusion prior objective for unseen views; $ \mathcal{L}_{\text {rec }}$ denotes the traditional reconstruction objective for reference views, which aligns the given image in depth space with Pearson correlation~\cite{sedgwick2012pearson} and the RGB, mask space through mean squared error (MSE)~\cite{marmolin1986subjective}; $\mathcal{L}_{\text {reg }}$ is the regularization term ensuring the geometrical smoothness.
By synergistically integrating these objectives, our DTC123 pipeline demonstrably achieves a high degree of geometric robustness and superior texture quality. 



%% file: sec/6_experiment.tex
\section{Experiments}\label{sec:exp}












\subsection{Implementation Details}

\label{detail}



\noindent \textbf{Pipeline Settings.} We consistently applied the \textit{same} set of hyper-parameters across all experiments. DTC123 was implemented in PyTorch with a single NVIDIA A100 GPU. We trained both the first and second stages for 3,000 iterations using the Adan~\cite{xie2022adan} optimizer with 1e-3 learning rate and 2e-5 weight decay, which cost approximately \textit{ 20-25} minutes for the entire pipeline. In the first stage, we adopted Zero123-XL~\cite{liu2023zero} (cfg=5) 
as the only teacher model to supervise an Instant-NGP~\cite{muller2022instant} with three MLP layers and hash-encoder. Then Stable Diffusion v2.0~\cite{rombach2022high}(cfg=25) was integrated with zero123-XL, jointly enhancing the geometric robustness and texture refinement of a DMTet, initialized from the prior stage. The reference view was sampled with a 25\% probability, and other views at 75\%, in both stages. Final rendering resolutions were set at 64 × 64 for the first stage and 512 × 512 for the second stage, respectively.




\noindent \textbf{Image Pre-processing Details.}
Given an arbitrary reference image, DTC123 pipeline systematically processed it for the follow-up 3D generation. The first step employed the state-of-the-art segmentation model, SAM~\cite{kirillov2023segany}, 
to meticulously distinguish foreground objects from their background context. Subsequently, the Dense Prediction Transformer~\cite{ranftl2021vision} was harnessed to estimate both depth and normal maps, ensuring that rich geometric information has been captured. 
Then, BLIP2~\cite{li2023blip} crafted a descriptive caption on the segmented object for text-conditioned guidance. The culmination of this workflow is a trio of outputs: a sharply segmented image, its corresponding depth and normal map, and its semantic description. 
Note that unlike~\cite{qian2023magic123}, ~\cite{melas2023realfusion},
textual inversion with image-level augmentations, which is time consuming (more than an hour), was not leveraged in ours.

\begin{figure}[h]
    \centering
    \includegraphics[width=\linewidth]{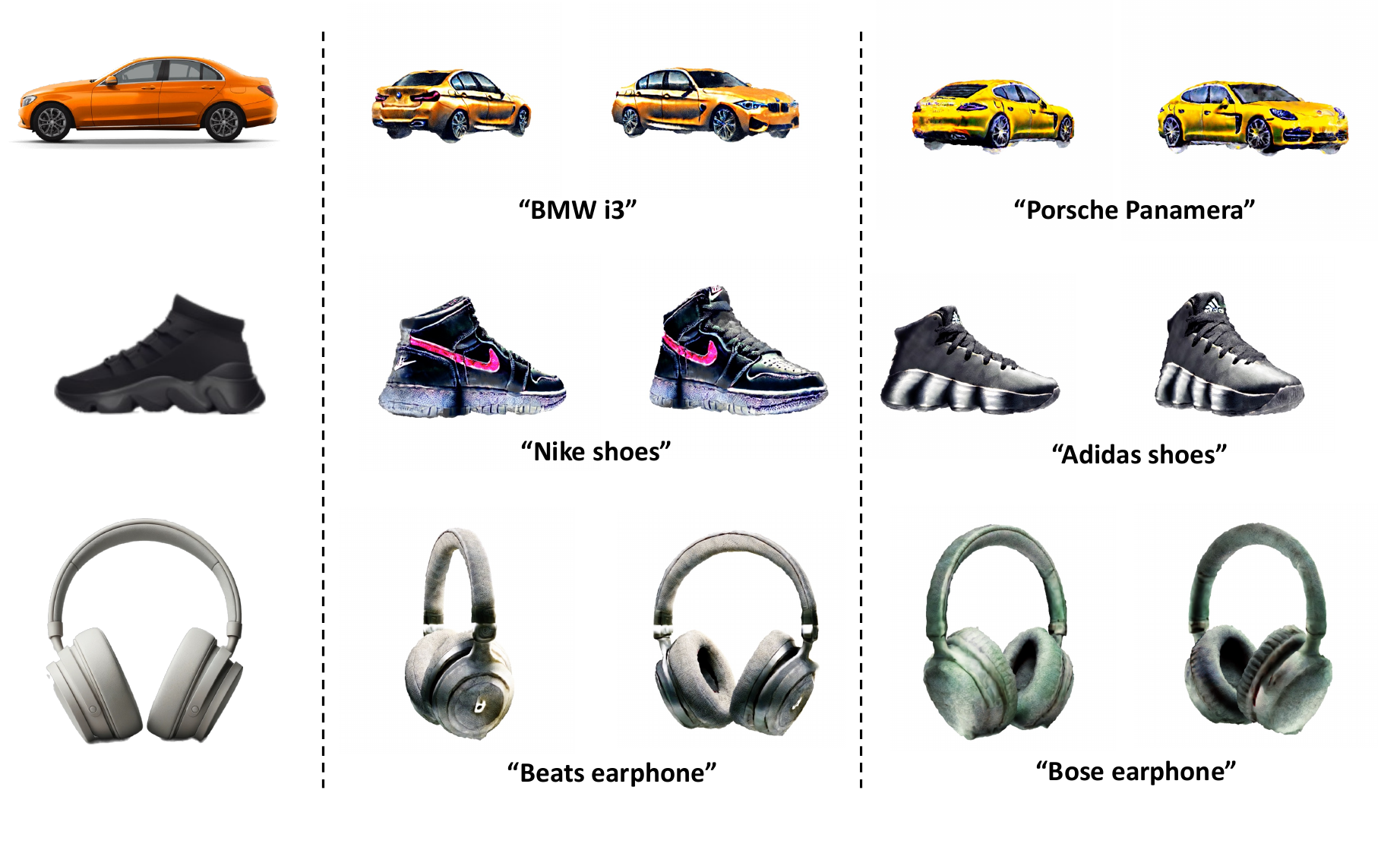}
    \caption{Multi-instance generation by customized prompts.}
    \vspace{-1mm}
    \label{multi}
\end{figure}

\begin{figure*}[t]
  \begin{minipage}[b]{1.0\linewidth}
  \centerline{\includegraphics[width=0.92\linewidth]{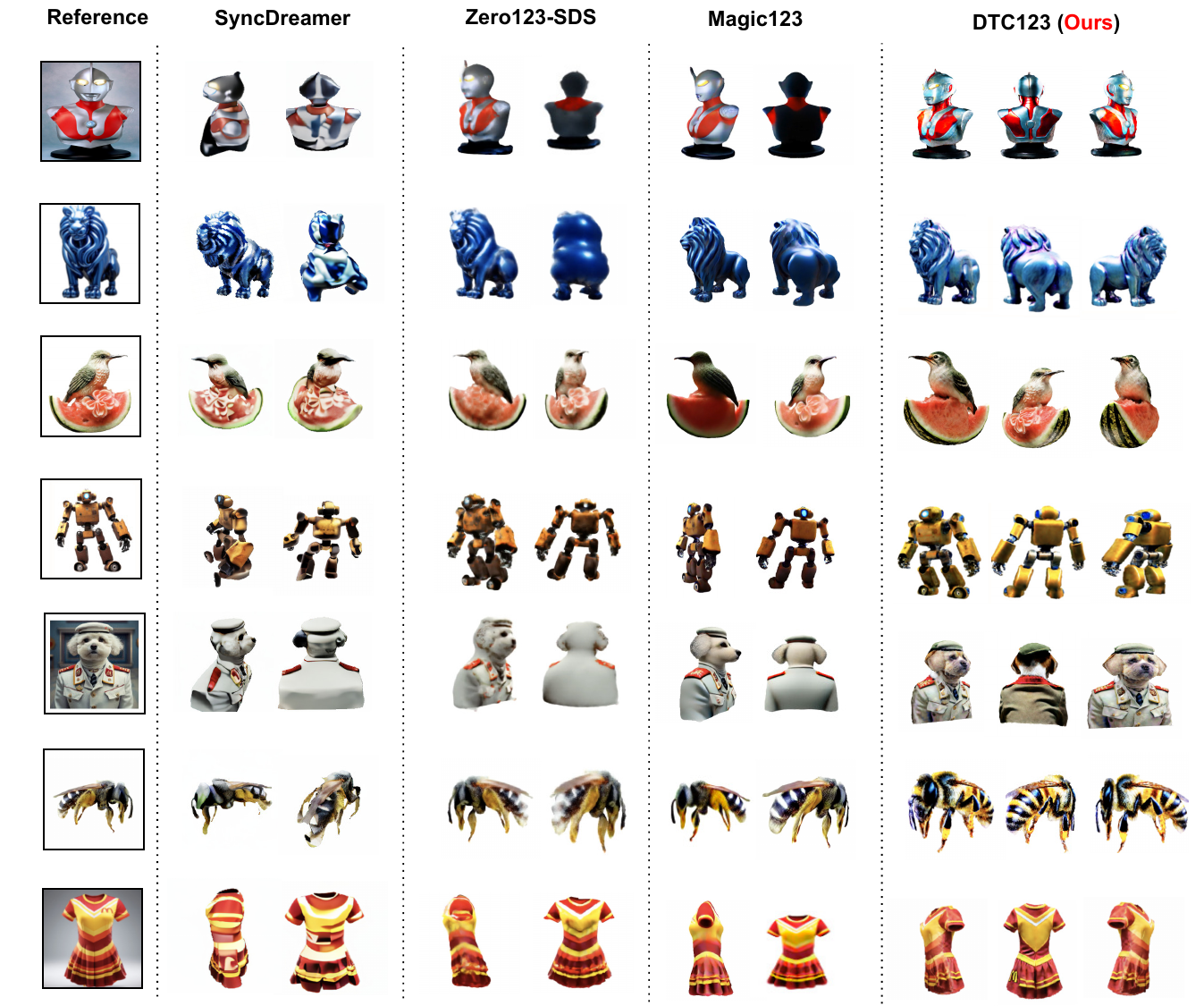}}
  \end{minipage}
  \caption{\textbf{Qualitative comparisons on image-to-3D generation}. We randomly sample several new views to present, while other views and methods are included in \textit{Appendix}. Our DTC123 consistently outperforms other state-of-the-art methods by generating multi-view consistent and high-fidelity results.} 
  \label{fig:5}
\end{figure*}

\subsection{Experimental Protocol}
An exceptional 3D model should not only mirror the reference image but also maintain a consistent correlation with the reference and plausible results when observed from other poses.
Following ~\cite{qian2023magic123}, we compared our DTC123 in RealFusion15~\cite{melas2023realfusion}, NeRF4 benchmark for quantitative comparison, with PSNR and LPIPS metrics to measure reconstruction quality and CLIP-similarity to evaluate appearance similarity. For qualitative results, we manually collected 50 reference images from the Internet out of the range of Objaverse, covering a wider range of difficult items. 

We adopted Zero123-SDS~\cite{liu2023zero}, RealFusion~\cite{melas2023realfusion}, NeuralLift~\cite{xu2023neurallift} and Magic123~\cite{qian2023magic123} as baseline SDS-based methods with their default experimental settings. We also compared with the state-of-the-art methods One-2-3-45~\cite{liu2023one} and SyncDreamer~\cite{liu2023syncdreamer}, which perform Image-3D generation in a feed-forward manner instead of SDS optimization. For Zero123-SDS and Magic123, we adopted the implementation from threestudio~\cite{threestudio2023} and leveraged Zero123-XL as ours for fair comparison, while others from their official codebase. We are quite confident that the baselines presented here are the finest re-implementations we have come across. 


\subsection{Image-3D Generation}

\noindent \textbf{Qualitative Results.}    
Figure~\ref{fig:5} demonstrates that DTC123 maintains high fidelity and plausible generation in complex scenarios. In contrast, generation results from most baselines, even those utilizing the more advanced Zero123-XL, are plagued by multi-view inconsistency, geometric distortion, and texture conflict. For example, in the yellow robot case in the \textit{fourth} line of Figure~\ref{fig:5}, competing textures cause blurring and exhibit unreasonable features, such as an extra eye on the back. In comparison, DTC123 generates high-fidelity novel views with realistic  metallic textures.

\noindent \textbf{{Quantitative Results.}}
Table~\ref{table:Image-to-3D} shows that DTC123 consistently outperforms other methods across all metrics, demonstrating its superior reconstruction (PSNR, LPIPS) and 3D consistency (CLIP-similarity) capabilities. Specifically, in the reference view reconstruction, DTC123 is on par with Magic123 and significantly exceeds RealFusion and NeuralLift. In terms of view consistency, as indicated by CLIP-similarity, DTC123 exceeds Magic123 by a large margin. The primary inconsistency in Magic123 stems from the boundary disparities between occluded and non-visible regions, resulting in pronounced seams. More results on GSO dataset is included in \textit{Appendix}.


\begin{table}[h]
\Huge
\centering
\caption{\textbf{Quantitative results.} We show quantitative results in terms of CLIP-Similarity$\uparrow$ / PSNR$\uparrow$ / LPIPS$\downarrow$. The results are shown on the NeRF4 and RealFusion datasets.}
\label{table:Image-to-3D}
\resizebox{\columnwidth}{!}{%
\begin{tabular}{@{}c|ccccccc@{}}
\toprule
\textbf{Dataset} & \textbf{Metrics\textbackslash{}Methods} 
& NeuralLift~\cite{xu2023neurallift}  & RealFusion~\cite{melas2023realfusion}
& Magic123~\cite{qian2023magic123} & \textbf{DTC123} \\
\midrule
\multirow{3}{*}{\textbf{NeRF4}}        
& CLIP-Similarity$\uparrow$  & 0.52 & 0.38 & 0.80 & 0.84 \\
& PSNR$\uparrow$             & 12.55 & 15.37 & 24.62 &  25.14 \\
& LPIPS$\downarrow$           & 0.50 & 0.20 & 0.03 & 0.02 \\ 
\midrule
\multirow{3}{*}{\textbf{RF15}} 
& CLIP-Similarity$\uparrow$ & 0.65 & 0.67 & 0.82 & 0.87 \\
& PSNR$\uparrow$            & 11.08 & 0.67 & 19.50 & 21.42 \\
& LPIPS$\downarrow$          & 0.53 & 0.14 & 0.10 & 0.08 \\ 
\bottomrule
\end{tabular}
}
\end{table}





\subsection{Multi-instance Generation}

While other Image-to-3D approaches~\cite{qian2023magic123,liu2023zero} can only generate plausible instances with limited diversity by different random seeds for initialization, DTC123 facilitates amazing multi-instance generation with different localized details by specific user prompt. For example, when the reference image presents the side view of a yellow car, users can specify their desired object, such as a 'BMW' or a 'Porsche Panamera'. DTC123 proficiently generates realistic and coherent 3D contents based on the given instructions, thanks to the appropriate time-step curriculum, where the student model captures the car's general shape initially, followed by brand-specific details. Please check Figure~\ref{multi} for additional results, which vividly demonstrates the potential of DTC123 for user-guided controllable generation and 3D editing.

\begin{figure*}[t]
    \centering
    \includegraphics[width=0.92\linewidth]{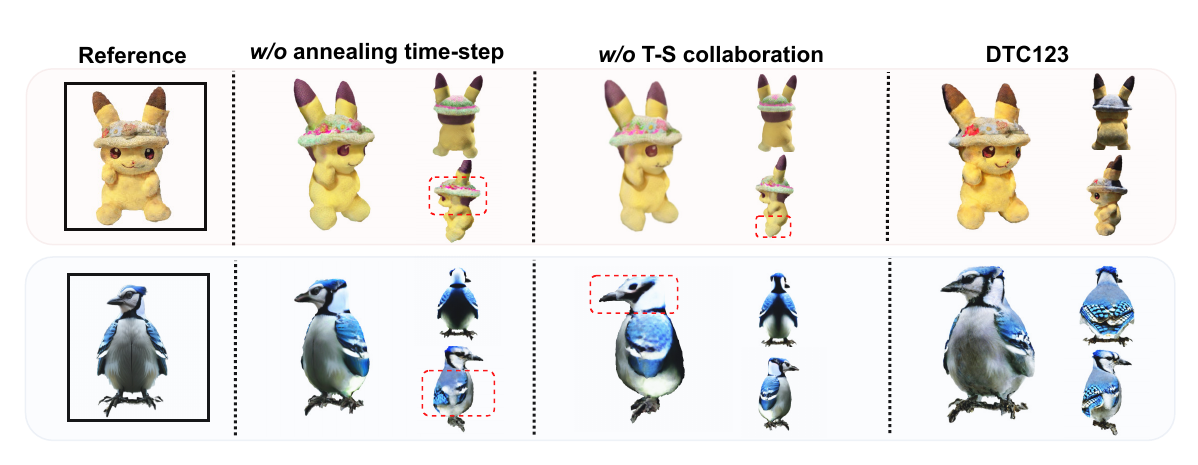}
        \vspace{-6mm}
    \caption{Ablation study on the component-wise contribution of DTC123. \textbf{T-S} denotes the \textbf{T}eacher-\textbf{S}tudent collaboration.}

    \label{ablation}
\end{figure*}



%% file: sec/7_discussion.tex
\section{Ablation and Discussion}




\noindent\textbf{Q1: }\emph{\textbf{What impacts performance of DTC123 in terms of component-wise contributions?}} We discarded each core component of DTC123 to validate its component-wise effectiveness. The results are depicted in figure~\ref{ablation}.

\noindent\textbf{A1: } In a multi-component pipeline, we observed that the exclusion of any component from DTC123 resulted in a significant degradation in performance. In particular, when the annealed sampling schedule is replaced with a random sampling schedule, it leads to geometric irregularities and excessive texture details. The similar situation occurs if teacher and student models don't collaborate with the annealed time-step. In contrast, when both teacher and student follow the time-step curriculum, it effectively generates results with sophisticated texture and illumination while ensuring geometric stability.

\noindent\textbf{Q2: }\emph{\textbf{How about the robustness of DTC123? }} To better diagnose the robustness of DTC123, we meticulously analyzed and quantified the occurrences of generation failure (\eg, Janus face, geometry distortion, extremely abnormal coloration) on Level50 with different random seeds and under different difficulty levels.

\noindent\textbf{A2: }
As depicted in Figure~\ref{figure5}, DTC123 consistently exhibits a lower failure rate compared to other methods across various initialization (random seeds) and at different difficulty levels by a large margin. Such robustness should be attributed to the proposed time-step curriculum, where the student 3D model initially captures coarse features, significantly reducing the instability typically associated with randomized NeRF initialization. 


\begin{figure}[h]
    \centering
    \includegraphics[width=0.95\linewidth]{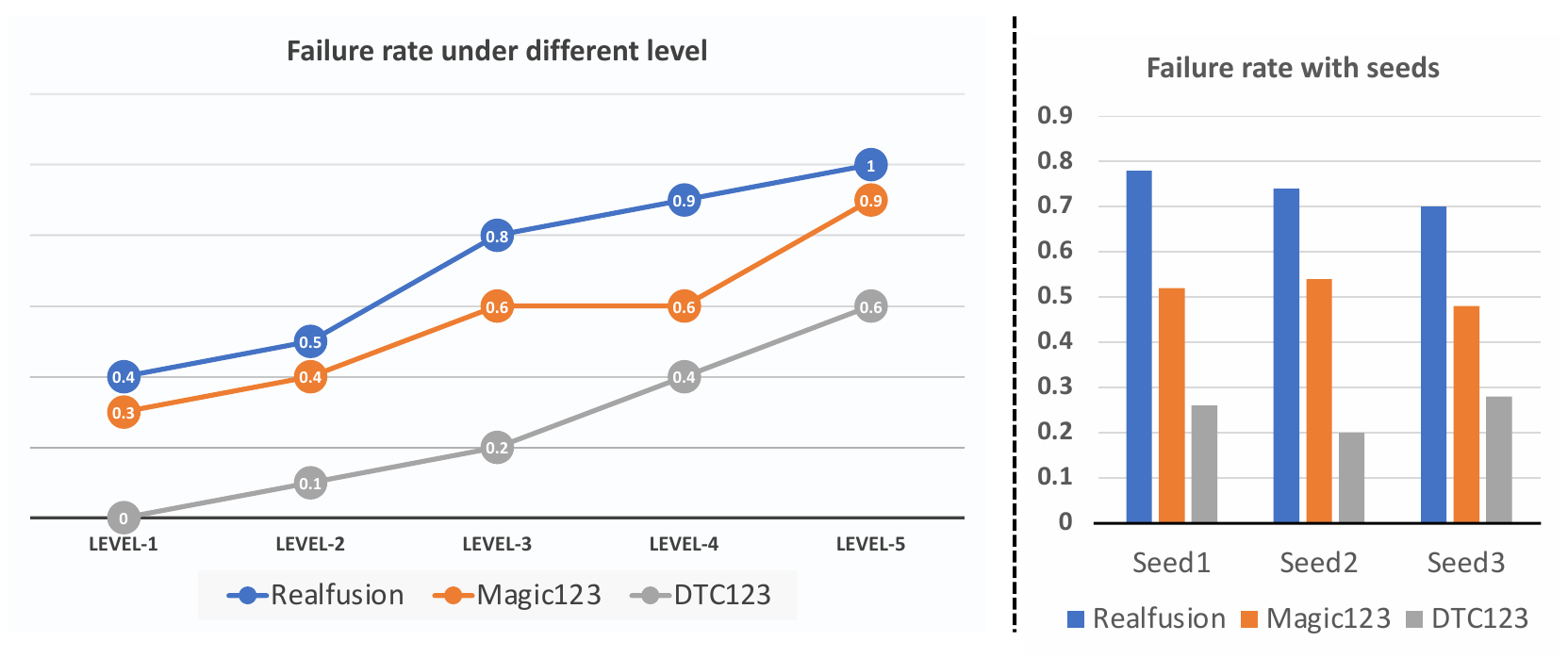}
    \caption{Failure analysis with different initialization and levels.}
    \label{figure5}
\end{figure}

\noindent\textbf{Q3: }\emph{\textbf{What's the difference between diffusion time-step curriculum and annealing time-step?}}

\noindent\textbf{A3:}
  We clarify that annealing time-step is \textbf{NOT} equivalent to, but a part of time-step curriculum. Our DTC includes three coherent parts : time-step schedule, progressive student representation and teacher guidance. Different from the viewpoint of DDPM sampling from Dreamtime~\cite{huang2023dreamtime}, we consider that SDS leverages teacher-generated $\hat{\bbx}_{0}$, which is a single-step de-nosing output with starting point $\bbx_{t}$, to optimize student-rendered $\bbx_{\pi}$. Modeling 3D generation as a data corruption reduction process, we reveal that there exists a time-step lower bound in SDS for \textit{teacher} diffusion model to well estimate the desired score function and thus provide quality guidance $\hat{\bbx}_{0}$. In order to alleviate the marked divergence of the teacher-generated $\hat\bbx_{0}$~\cite{ho2020denoising} at large time-steps, we further design the progressive student representation and teacher guidance to cooperate with the annealing time-step for a more stable coarse-to-fine generation. Please refers to \textit{Appendix} for more details.






%% file: sec/8_conclusion.tex
\section{Conclusion}\label{sec:con}
We revisit Score Distillation Sampling (SDS) and point out that the crux of its enhancement lies in the proposed diffusion time-step curriculum. We then design an improved coarse-to-fine Image-3D pipeline (DTC123) which collaborates the teacher diffusion model and 3D student model with the time-step curriculum. Through qualitative comparisons and quantitative evaluations, we show that our DTC123 significantly improves the photo-realism and multi-view consistency of Image-to-3D generation. In the future, we will focus on exploring the potential of diffusion time-step curriculum with advanced teacher models and diverse student models to further improve the efficiency and quality.



\section*{Acknowledgments}
DTC123 is supported by the National Research Foundation, Singapore under its AI Singapore Programme (AISG Award No: AISG2-RP-2021-022) and the Agency for Science, Technology AND Research. Pan Zhou is supported by the Singapore Ministry of Education (MOE) Academic Research Fund (AcRF) Tier 1 grant.

%% file: sec/X_suppl.tex
\maketitlesupplementary
\appendix


\noindent The \textit{Appendix} is organized as follows:
\begin{itemize}[leftmargin=*]
    \item \textbf{Section~\ref{a}:} gives a theoretical justification of the proposed \textit{Diffusion Time-step Curriculum}.
    \item \textbf{Section~\ref{b}:} further provides the experimental justification and discussion on the collaboration of the teacher and student with \textit{Diffusion Time-step Curriculum}.
    \item \textbf{Section~\ref{c}:} elaborates more details about our DTC123 pipeline. Specifically, we detailed the implementation of instruct-LLM design and geometry smoothness regularization.
    \item \textbf{Section~\ref{d}:} showcases more immerse experiment results and ablation studies with Level50 benchmark.
\end{itemize}

\section{Theoretical Justification}
\label{a}
This section elaborates on our  \textit{diffusion time-step curriculum} motivation, where larger time steps capture coarse-grained concepts and smaller time steps learn fine-grained details. We first show that a diffusion time-step curriculum is necessary which further induces an annealed sampling strategy. Upon that, we explain why teacher and student models should collaborate with each other to achieve such a time-step curriculum.

Specifically, SDS employs the de-noised $\hat\bbx_{0}$ 
generated by the teacher diffusion model  $\boldsymbol{\epsilon_{\phi}}\left(\bbx_{t} ; \boldsymbol{y}, t\right)$  to guide the  student-rendered  $\bbx_{\pi}$. 
Thus, the crux of this teacher-student optimization process is determined by the quality of the teacher guidance $\hat\bbx_{0}$, which motivates us to explore an appropriate strategy to ensure the valid guidance $\hat\bbx_{0}$ during any training iterations. We first formalize the definition of our target from the perspective of student data corruption.

\begin{definition}
\textup{(Data Corruption Reduction)} Given the camera pose $\pi$ and the condition $y$, we consider a student-rendered image $\bbx_{\pi} = g(\theta, \pi)$ at an arbitrary training iteration $k$. Suppose that there exists a real data point $\bbx^*$ drawn from the data distribution $p_{\text{data}}(\bbx)$ and an unknown data corruption $\delta_k = \KL[\delta(\bbx - \bbx_{\pi}) || p_{\text{data}}(\bbx)]$, such that we can express $\bbx_{\pi}$ as $\bbx_{\pi} = \bbx^* + \delta_k$. Our \textbf{objective} is to iteratively reduce the corruption $\delta_k$ inherent in $\bbx_{\pi}$ as $k \to \infty$, guided by the conditional score function $\nabla \log p_{t}(\bbx)$, such that $\bbx_{\pi}$ increasingly resembles a sample from $p_{\text{data}}(\bbx)$.
\end{definition}

\begin{figure}[t]
   \begin{minipage}[b]{1.0\linewidth}
   \centerline{\includegraphics[scale = 0.52]{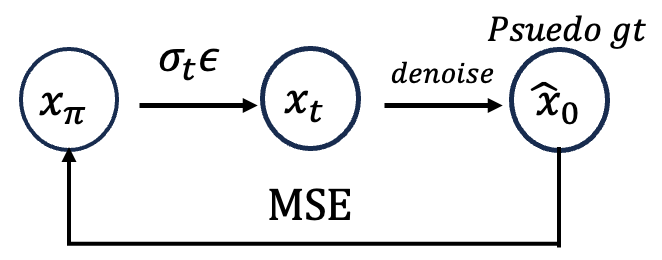}}
   \end{minipage}

\caption{ Simplified illustration of SDS.}
   \label{appappappfig:2}
\end{figure}


To estimate the real sample $\bbx^*$ in  $\bbx_{\pi}$ for good guidance, one can resort to the score function $\nabla \log p_{t}(\bbx)$, which, however, is practically inaccessible. To solve this issue, a pre-trained diffusion model $\boldsymbol\epsilon_{\phi}(\bbx, t)$ 
is often used to estimate the score function as shown in previous works~\cite{graikos2022diffusion,wang2023not,ho2020denoising}.  But to better denoise and thus produce quality $\hat{x_{0}}$, the teacher model $\boldsymbol\epsilon_{\phi}(\bbx, t)$ needs a certain time step $t$ to inject noise $\sigma_{t} \boldsymbol{\epsilon}$ into $\bbx_{\pi}$. In this way, the noisy sample $\bbx_{t} = \bbx_{\pi} 
+\sigma_{t} \boldsymbol{\epsilon}$ can approximately lie in the forward diffusion distribution $p_{t}(\bbx_t)=  \int_{\mathbb{R}^d} \mathcal{N}(\bbx, \sigma_t^2 \rmI) p_{\text{data}}(\bbx) \mathrm{d}\bbx$ in the diffusion model $\boldsymbol\epsilon_{\phi}(\bbx, t)$, \ie, the marginal distribution at time-step $t$ in the forward diffusion process which satisfies $p_{0}(\bbx_0) =   p_{\text{data}}(\bbx)$.  
For this point, we provide a formal analysis, and derive a \textit{proper} time-step sampling for better diffusion de-noising.  

\begin{figure*}[t]
    \centering
    \includegraphics[width=\linewidth]{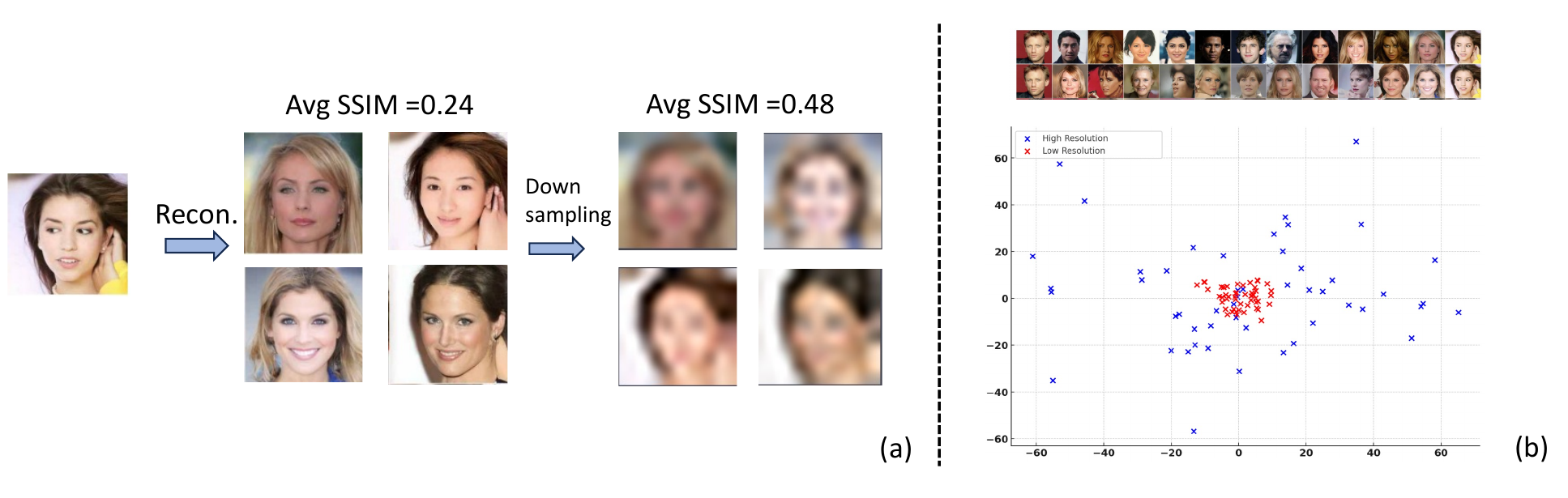}
    \caption{\textbf{(a)} The low-resolution generated set embraces similarity compared to its high-resolution counterparts. \textbf{(b)} The T-SNE visualization of high/low resolution generated sets on CelebA dataset.}
    \label{appfig:1}
\end{figure*}

 \begin{theorem}
\label{thm:lower_bound}
\textup{(Diffusion Time-step Lower bound)} Assume \(p_{t}(\bbx)\)  is the noisy data distribution and 
\(q_t(\bbx_t|\bbx_{\pi}) = \gN(\bbx_t; \alpha_t \bbx_{\pi}, \sigma_t^2 \rmI)\), for any $\bbx_t \sim q_t(\bbx_t|\bbx_{\pi})$, we have 
$\lVert \boldsymbol{\epsilon}_{\phi}(\bbx_t, t) - \nabla \log p_t(\bbx_t) \|_2^2 = \mathcal{O}(\varepsilon)$, if these two conditions hold :\\
a) the pretrained teacher diffusion model  $\boldsymbol{\epsilon}_{\phi}(\bbx, t)$ satisfies $\lVert \boldsymbol{\epsilon}_{\phi}(\bbx, t) -\nabla \log p_{t}(\bbx) \rVert_2^2 < \varepsilon$;\\
b) $t \sim \mathcal{U}[\Tilde{T}_{\delta_k, \varepsilon}, T]$ where $\Tilde{T}_{\delta_k, \varepsilon} = \mathcal{O}(\frac{\|\delta_k\|}{\epsilon})$.
\end{theorem}

\begin{proof}
We assume a) always hold since $\boldsymbol{\epsilon}_{\phi}(\cdot)$ is a well trained diffusion model.
Assume that the diffusion model $\boldsymbol{\epsilon}_{\phi}(\cdot)$ satisfies the Lipschitz condition with a constant $L > 0$. Specifically, we have:
\begin{equation}
    \| \boldsymbol{\epsilon}_{\phi}(\bbx, t) - \boldsymbol{\epsilon}_{\phi}(\bbx^{\prime}, t) \| \leq L \| \bbx - \bbx^{\prime} \| .
\end{equation}
From condition a) we have for any $t$:
\begin{equation}
    \lVert \boldsymbol{\epsilon}_{\phi}(\bbx, t) - \nabla \log p_{t}(\bbx) \rVert_2^2 < \varepsilon.
\end{equation}
Recall that $\bbx_{\pi} = \bbx^* + \delta_k$ where $\bbx^* \sim p_{\text{data}}(\bbx)$, we have:
\begin{equation}
\begin{aligned}
& \| \boldsymbol{\epsilon}_{\phi}(\alpha_t \bbx_{\pi} + \sigma_t \epsilon, t) - \nabla \log p_{t}(\alpha_t\bbx^* + \sigma_t \epsilon) \| \\
\leq  &  
\| \boldsymbol{\epsilon}_{\phi}(\alpha_t\bbx_{\pi} + \sigma_t \epsilon, t) - \boldsymbol{\epsilon}_{\phi}(\alpha_t\bbx^* + \sigma_t \epsilon, t) \| \\
+ & \| \boldsymbol{\epsilon}_{\phi}(\alpha_t\bbx^* + \sigma_t \epsilon, t) - \nabla \log p_{t}(\alpha_t\bbx^* + \sigma_t \epsilon) \| \\
\leq & L \alpha_t \|\delta_k\| + \epsilon
\end{aligned}
\end{equation}
To achieve the desired accuracy, we have to let $t$ sufficient large such that $L \alpha_t \|\delta_k\| = \mathcal{O}(\epsilon)$, which means $\alpha_t = \mathcal{O}(\frac{\epsilon}{L \|\delta_k\|})$. Recall that in conventional diffusion models~\cite{ho2020denoising,karras2022elucidating,2021scorebased}, we have $\alpha_t \propto \frac{1}{t}$, thus we derive 
\begin{equation}
    t \geq \mathcal{O}(\frac{L \|\delta_k\|}{\epsilon}) = \mathcal{O}(\frac{\|\delta_k\|}{\epsilon}).
\end{equation}

\end{proof}


Theorem~\ref{thm:lower_bound} shows that the teacher diffusion model can accurately estimate the desired score function $\nabla \log q_t(\bbx_t|\bbx_{\pi})$ under the condition a) and b). For condition a), it often holds, since the pretrained teacher diffusion model $\boldsymbol{\epsilon}_{\phi}(\bbx, t)$ can (approximately) converge to the forward diffusion distribution $p_{t}(\bbx)$. Thus, if one can sample a proper time step $t$ such that $\bbx_t$ satisfies condition b), then the teacher diffusion model $\boldsymbol{\epsilon}_{\phi}(\bbx, t)$  can well denoise $\bbx_t$ and provides quality guidance $\hat{\bbx}_{0}$ to supervise the student 3D model. Since in the early training iterations, the student-rendered $\bbx_{\pi}$ contains high corruption $\delta_k$ and  $\Tilde{T}_{\delta_k, \varepsilon}$ positively depends on the corruption level as shown in Theorem~\ref{thm:lower_bound}, a large time step $t\geq \Tilde{T}_{\delta_k, \varepsilon}$  is needed to inject more noise $\sigma_{t} \epsilon$ into $\bbx_{\pi}$ so that condition b) holds and thus guarantees the quality of the denoising.

However, there is a dilemma that the marked divergence of the teacher-generated $\hat\bbx_{0}$~\cite{ho2020denoising} at large time-steps could negatively compromise the student coherent modeling that possesses consistent geometric and photometric properties, resulting in geometry distortion and mode collapse~\cite{tang2023dreamgaussian,wang2023prolificdreamer}.
From the perspective of information theory~\cite{ash2012information}, given a set of teacher-generated $\hat\bbx_{0}$ in certain training iterations, coarse-grained information (\eg, blur contour) tends to have less variance than its fine-grained counterpart (\eg, texture nuances), which is also empirically justified in Part~\ref{b}.
This motivates us to first focus on fundamental, low-variance student-teacher knowledge transfer with large time-steps. 
     As the coarse-grain converges along with the training iteration $k$, the corruptions $\delta_k$ in the student-rendered $\bbx_{\pi}$ diminish. From Theorem~\ref{thm:lower_bound}, smaller noise can counteract corruption $\delta_k$ and help $\bbx_{t}$ conform to distribution $p_{t}(\bbx_t)$. Thus, the teacher diffusion model often only refines $\bbx_{t}$ to improve the fine-grains without destroying the course-grains. Accordingly, we can gradually derive a more accurate estimation of the score function $\nabla \log p_{t}(\bbx)$, thus ensuring the refinement of intricate details (\eg, texture nuances) at smaller time-steps. 


\begin{figure*}[t]
    \centering
    \includegraphics[width=\linewidth]{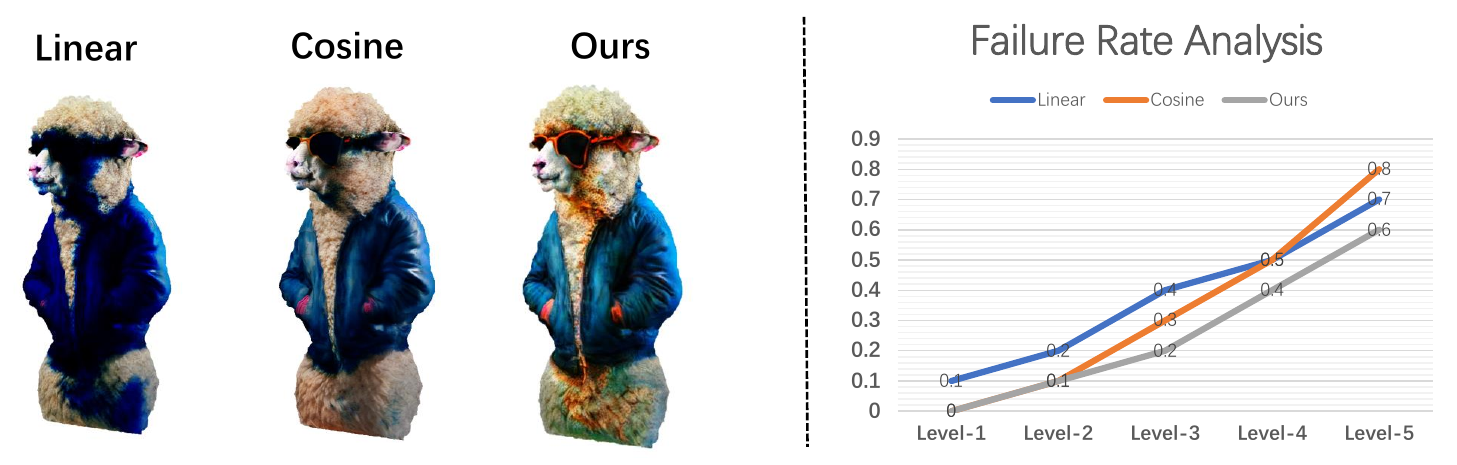}
    \caption{Ablations on different time-step sampling strategy.}
    \label{fig:3}
\end{figure*}

\section{Experimental Justification and Discussion}
\label{b}
This section elaborates on the experimental justification for the collaboration of the teacher and student model with the \textit{diffusion time-step curriculum}. We first intuitively indicate that the low-resolution representation has lower variance and is more robust compared to its high-resolution counterpart. Then we quantitatively analyze the view-conditioned and text-conditioned teacher diffusion model by comparing their reconstructed results with the multi-view ground-truth rendering images among different levels of perturbed noise.

\noindent \textbf{Student progressive representation.} We first generated a set of diverse teacher guidance $\hat{\bbx_{0}}$, following ~\cite{ho2020denoising} to use the CelebA~\cite{liu2015faceattributes} images by perturbing them with large noise (1000 diffusion time-steps) and reconstructing them during the reverse process. We then leveraged image down-sampling to simulate low-resolution modeling. As depicted in Figure~\ref{appfig:1}(a), the down-sampled generated set is more similar to each other than its high-resolution counterpart: the high-resolution images scored a structural similarity index (SSIM) of \textit{0.24}, contrasting sharply with the low-resolution set's \textit{0.48}. This significant difference highlights the reduced variance and increased similarity among down-sampled images, which is more suitable for student coherent modeling in the initial iteration. Figure~\ref{appfig:1}(b) further shows the T-SNE visualization of the high/low resolution generated set, which vividly demonstrates the low-variance and clustering effect of low-resolution representation. This toy experiment indicates that though the teacher diffusion model provides diverse guidance with a large time step~\cite{ho2020denoising}, we can ensure coherent modeling of 3D models by coarse-to-fine representation (resolution constraint) with the annealed time-step.

On the other hand, from the perspective of parameter sensitivity of multi-resolution modeling~\cite{muller2022instant}, rays intersecting the scene at coordinates \( x \) and the hash function defined as \( h(x) = \bigoplus_{i=1}^{d} x_i \pi_i \mod T \), even minor deviations in \( x \) lead to notable variations in encoded values. Denoting \( \delta x \) as the noise-induced deviation, the change in hash values, \( \Delta h \), escalates with higher resolutions, i.e.,
$\frac{\partial \Delta h}{\partial \delta x} \propto \text{resolution}$, indicating that
higher resolutions inherently magnify the sensitivity to noise and variance in the ground truth, which also inspires us to first capture the low-resolution representation for more stable 3D model optimization.

\noindent \textbf{Teacher coarse-to-fine prior.} 
Here, we conducted a quantitative experiment to answer the question about the suitable teacher for diffusion time-step curriculum. Given a 3D object from the \textit{high-quality real-scanned} dataset OminiObject3D~\cite{wu2023omniobject3d}, we used Cycles Engine in Blender to randomly (with fixed elevation and different azimuth ) render 16 multi-view images and transparent backgrounds with pure grey color. We then perturb them with different levels of noise and then compare~\footnote{Note that the condition of Zero-1-to-3 is the default front view image and the camera parameters, while the condition of Stable Diffusion is the caption of the 3D object.} the quality of the reconstructed image with the ground-truth renderings by investigating the contour exploration consistency via MaskIoU and the perceptual generation quality computed by CLIP-similarity.

As illustrated in Table 1, (1) At large time-steps, Zero-1-to-3 generated outputs tend to have better MaskIoU than Stable Diffusion outputs, which suggests that Zero-1-to-3 serves as a coarse-grained teacher by providing a more accurate contour or boundary at large $t$; (2) At smaller time-steps, both Zero-1-to-3 and Stable Diffusion have relatively high MaskIoU since the small scale perturbed noise doesn't corrupt the overall geometry structure. Considering the quality of perceptual generation (CLIP similarity), we notice that Stable Diffusion surpasses Zero-1-to-3 to some extent, indicating that Stable Diffusion is more suitable for a fine-grained teacher, since it produces more realistic texture details at smaller $t$.

\begin{table*}[h]
\centering
\renewcommand{\arraystretch}{1.2} 
\setlength{\tabcolsep}{5pt}       
\begin{tabular}{|>{\bfseries}c|*{2}{c|}|*{2}{c|}|*{2}{c|}|*{2}{c|}}
\hline
\rowcolor{mylightgray}
\multicolumn{9}{|c|}{\textbf{Quantitative analysis among different time-steps}} \\
\hline
\rowcolor{mylightgray}
\textbf{Time-step} & \multicolumn{2}{c|}{\textbf{200}} & \multicolumn{2}{c|}{\textbf{400}} & \multicolumn{2}{c|}{\textbf{600}} & \multicolumn{2}{c|}{\textbf{800}} \\
\hline
\rowcolor{mylightgray}
\textbf{Metrics} & \textbf{MaskIoU} & \textbf{CLIP-S} & \textbf{MaskIoU} & \textbf{CLIP-S} & \textbf{MaskIoU} & \textbf{CLIP-S} & \textbf{MaskIoU} & \textbf{CLIP-S} \\
\hline
\textbf{Zero-1-to-3} & 0.92 & 0.84 & 0.87 & 0.82 & 0.84 & 0.79 & 0.82 & 0.80 \\
\hline
\textbf{Stable Diffusion} & 0.89 & 0.90 & 0.81 & 0.84 & 0.72 & 0.82 & 0.63 & 0.74 \\
\hline
\end{tabular}
\caption{Quantitative compassion of Zero-1-to-3 and Stable Diffusion among different time-steps, where CLIP-S denotes the CLIP similarity between de-noising output and the ground truth renderings. }
\end{table*}

\section{Background}
\label{c}

Due to the limited space of the camera-ready version, we mainly introduce \textit{student} 3D models, \textit{teacher} diffusion models here that will help to build up our approach.

\subsection{Student 3D Model}

We aim to learn an underlying 3D representation $\theta$ (\eg, NeRF, mesh), which uses a differentiable renderer $g(\cdot)$ to generate the relative image from any desired camera pose $\pi$ by $\bbx_{\pi} = g(\theta, \pi)$. 
For computation memory concerns, we leverage NeRF~\cite{mildenhall2021nerf} for low-resolution scene modeling, and then adopt DMTet~\cite{shen2021deep} for high-resolution mesh fine-tuning.

\noindent\textbf{(NeRF)~\cite{mildenhall2021nerf}} is a differentiable volumetric representation. It characterizes the scene as a volumetric field by density and color with a neural network $\theta$. 
Given a camera pose $\pi$, the rendered image can be computed by alpha compositing the color density field.
Considering rendering efficiency, multi-resolution hash grids~\cite{muller2022instant} are usually utilized to parameterize the scene. This representation helps to achieve high-quality rendering results with a faster training speed.

\noindent\textbf{Hybrid SDF-Mesh Field (DMTET)} is a differentiable surface representation. It parameterizes the Signed Distance Function (SDF) by a deformable tetrahedral grid $(V_T, T)$, where $T$ represents the tetrahedral grid and $V_T$ corresponds to its vertices. By assigning every vertex $v_i \in V_T$ with a SDF value $s_i \in R$ and a deformation vector $\Delta v_i \in R^3$, this representation allows recovering a explicit mesh through differentiable marching tetrahedra.

\subsection{Teacher Diffusion Model} 

Conditional Diffusion Model (CDM)~\cite{ho2022classifier,kawar2022enhancing,wu2023fast} basically generates the desired samples given a certain condition $y$.
In the context of SDS-based 3D generation, we mainly utilize the following two types of teacher diffusion prior with different conditions:

\noindent \textbf{Text-conditioned Prior.} Large-scale pre-trained text-to-image diffusion models, \eg, Stable Diffusion~\cite{rombach2022high}, are often leveraged with the text description condition of the reference image. In practice, one often employs textual inversion~\cite{gal2022image,yang2023controllable} or large vision-language models \eg, BLIP2~\cite{li2023blip} to generate the text description of the reference image. 

\noindent \textbf{View-conditioned Prior.} Zero-1-to-3 ~\cite{liu2023zero} 
 is fine-tuned from the Stable Diffusion image variations~\cite{LambdaLabs2023} on the 3D synthetic dataset Objaverse~\cite{deitke2023objaverse}, and integrates viewpoint control to conduct novel view synthesis given the camera pose and reference image as condition.

\section{Implementation Details}
\label{d}

This section elaborates on the details of implementation details of instruct-LLM design and geometry smoothness regularization.

\noindent \textbf{Instruct-LLM Design.} In the second stage, we noted that employing finer, more precise linguistic prompts to significantly narrow the image distribution~\cite{bai2023integrating} in Stable Diffusion, complementing the mode-seeking SDS algorithm. Utilizing large language models (LLMs), we converted BLIP2-derived prompts into comprehensive, view-specific descriptions. This approach intensifies detail in each orthogonal view, avoiding superfluous structural descriptions and resolving perspective conflicts.
The detailed instruction is as follows:

\begin{custombox}
 ``I am about to begin a series of image-3D generation tasks and need your help to create prompt descriptions. I'll provide the frontal description first. You will then give a one-sentence description for each the left, right, and rear sides, ensuring: (1) Alignment with the frontal description, (2) Conciseness with rich textural detail, (3) 3D consistency across descriptions, using DALLE-3 for validation.''
\end{custombox}

Figure~\ref{appappfig:2} illustrates the prompts created using this method, demonstrating our technique's effectiveness in generating consistent, detailed multi-view descriptions.
\begin{figure}[h]
    \centering
    \includegraphics[width=\linewidth]{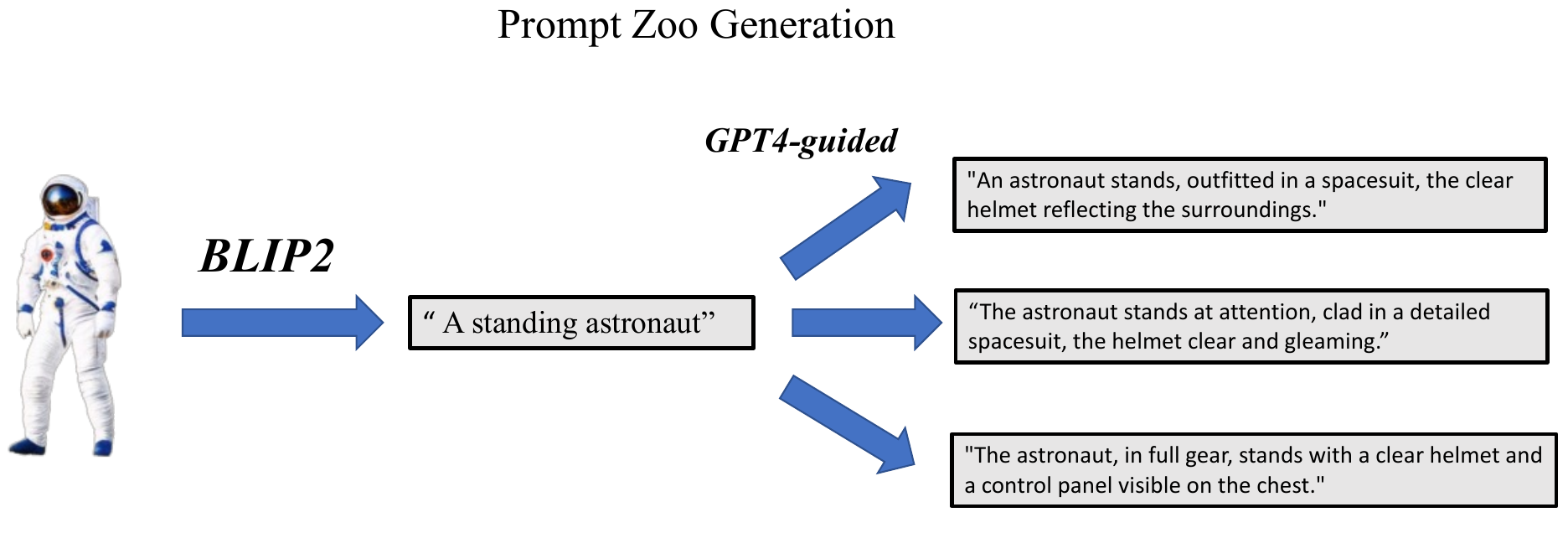}
    \caption{Examples of viewpoint-augmentation with LLM. Given an abstract prompt, our instruct-LLM outputs complement the initial prompts with insufficient information.}
    \label{appappfig:2}
\end{figure}

\noindent \textbf{Geometry smoothness regularization.} We observe that the 3D model occasionally generated high-frequency artifacts on the crisped surface and edge contour. Following \cite{stable-dreamfusion}, we leverage the normal vector regularization $\mathcal{L}_{\text{reg}}$:

\begin{equation}
\mathcal{L}_{\text{reg}} = \mathbb{E}_{\mathbf{a}}\left[\| \mathbf{n}(\mathbf{a}) - \mathbf{n}\left(\mathbf{a} + \beta \cdot \mathcal{N}(0, I)\right) \|_1 \right]
\end{equation}

where $\mathbf{n}$ denotes the normal vector at a point $\mathbf{a}$ in the 3D space, $\beta $ is a small perturbation scale and $\mathcal{N}(0, I)$ is standard Gaussian noise. 

In the second stage of DMTet, we implement Laplacian smoothing regularization where the Laplacian matrix $\mathbf{L} \in \mathbb{R}^{V \times V}$ is computed by identifying adjacent vertex pairs from each face in $\mathbf{F}$. Entries in $\mathbf{L}$ are set to $-1$ for adjacent vertices and to the degree of the vertex on the diagonal. The smoothing loss $\mathcal{L}_{\text{reg}}$ is defined as the mean norm of the product of the Laplacian matrix $\mathbf{L}$ and vertex positions $\mathbf{V}$:

\begin{equation}
\mathcal{L}_{\text{reg}}  = \text{mean}\left(\|\mathbf{L} \cdot \mathbf{V}\|_2\right),
\end{equation}

which ensures the uniformity and smoothness of the mesh by minimizing deviations in the vertex positions, leading to a more regular and smooth 3D structure.

\section{More Experimental Results}
\label{e}

This section presents more qualitative and quantitative results of DTC123 in both text-to-3D and Image-to-3D generation tasks and ablation studies on the time-step sampling strategy.

\subsection{Text-to-3D Generation Results}
For the text-to-3D tasks in Figure~\ref{appfig:3}, we adapt our pipeline by replacing Zero-1-to-3 with MVdream for coarse guidance, leveraging different time-step schedule and prompt de-biasing. As illustrated in Figure~\ref{appfig:3}, our DTC123 can generate high-fidelity 3D assets corresponding to the given text prompt.

\begin{figure}[h]
   \begin{minipage}[b]{1.0\linewidth}
   \centerline{\includegraphics[scale = 0.28]{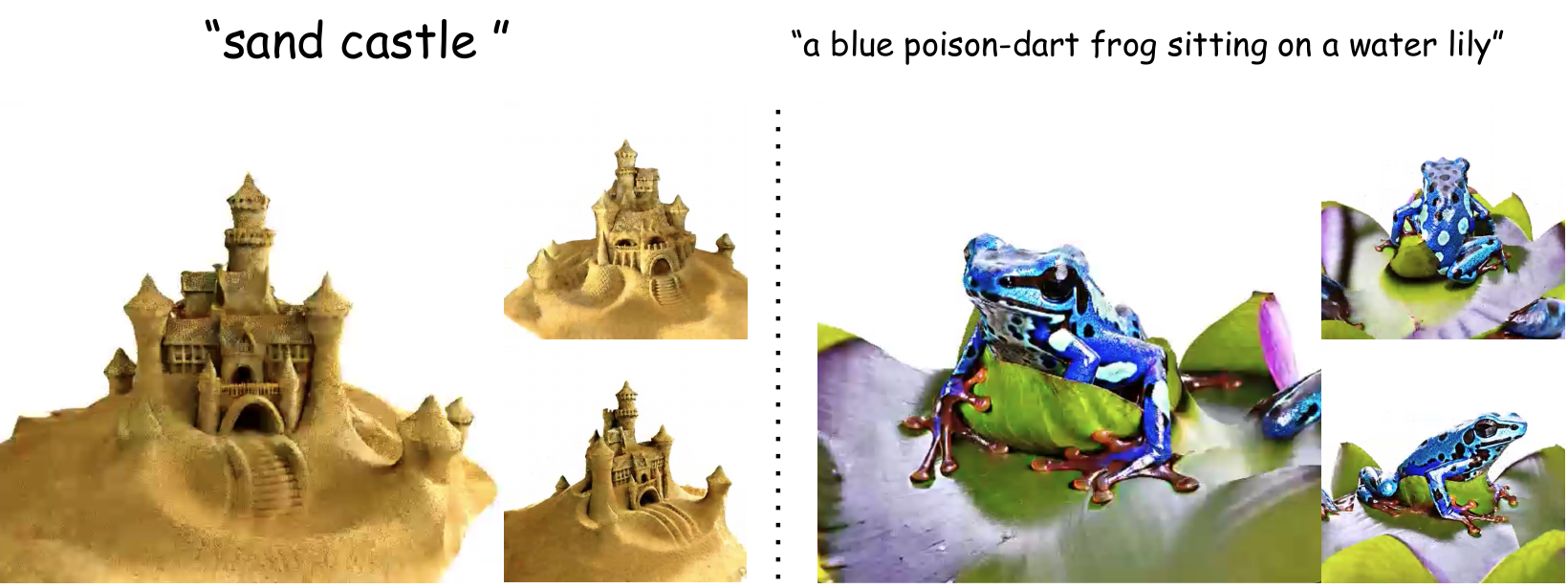}}
   \end{minipage}
\captionsetup{font={scriptsize}} 
\vspace{-6mm}
\caption{\scriptsize Text-to-3D Results with diffusion time-step curriculum.}
   \label{appfig:3}
\end{figure}

\subsection{Ablation on Time-step Sampling Strategy}
To better analyze the DTC123 premium sampling strategy, we still adopted the failure rate metric in the manuscript on Level50.
As illustrated in Figure~\ref{fig:3}, our proposed sampling strategy consistently exhibits a lower failure rate compared to other methods at different difficulty levels. We justify that the robustness should be attribute to the introduced local randomness with the annealed interval, since
(1) even within the same training iteration, the corruption level of 3D models varies across camera poses, and (2) in contrast to~\cite{wang2023not}, it is nearly impossible to pinpoint the exact corruption level without the ground-truth of unseen view, we need some randomization for self-calibration of the teacher-student symbiotic cycle, like SDE sampling in conventional DMs~\cite{karras2022elucidating}, which can partially correct the cumulative error introduced by the optimization process and alleviate `floaters' effects.

\begin{table}[h]
\centering
\caption{Quantitive Results on GSO dataset.}
\setlength{\tabcolsep}{4pt} 
\renewcommand{\arraystretch}{1} 
\begin{tabular}{lccc}
\hline
\textbf{Metric} & \textbf{Zero123} & \textbf{Magic123} & \textbf{DTC123} \\ \hline
Chamfer Distance ↓ & 0.110 & 0.129 & 0.028 \\
volumetric IoU ↑   & 0.382 & 0.327 & 0.610 \\ \hline
\end{tabular}
\label{tab}
\end{table}

\subsection{More Quantitative Results }
Table~\ref{tab}, we randomly chose 10 objects from GSO dataset for multi-view elevation, which vividly showcased that our method outperforms Magic123 and Zero-1-to-3 by a large margin.

\subsection{More Qualitative Results }
In Figure~\ref{fig:big} and Figure~\ref{fig:big2}, we present additional qualitative outputs with high fidenity and multi-view consistency. Please check out the video demos for more results.

\begin{figure*}[t]
    \centering
    \includegraphics[width=0.95\linewidth]{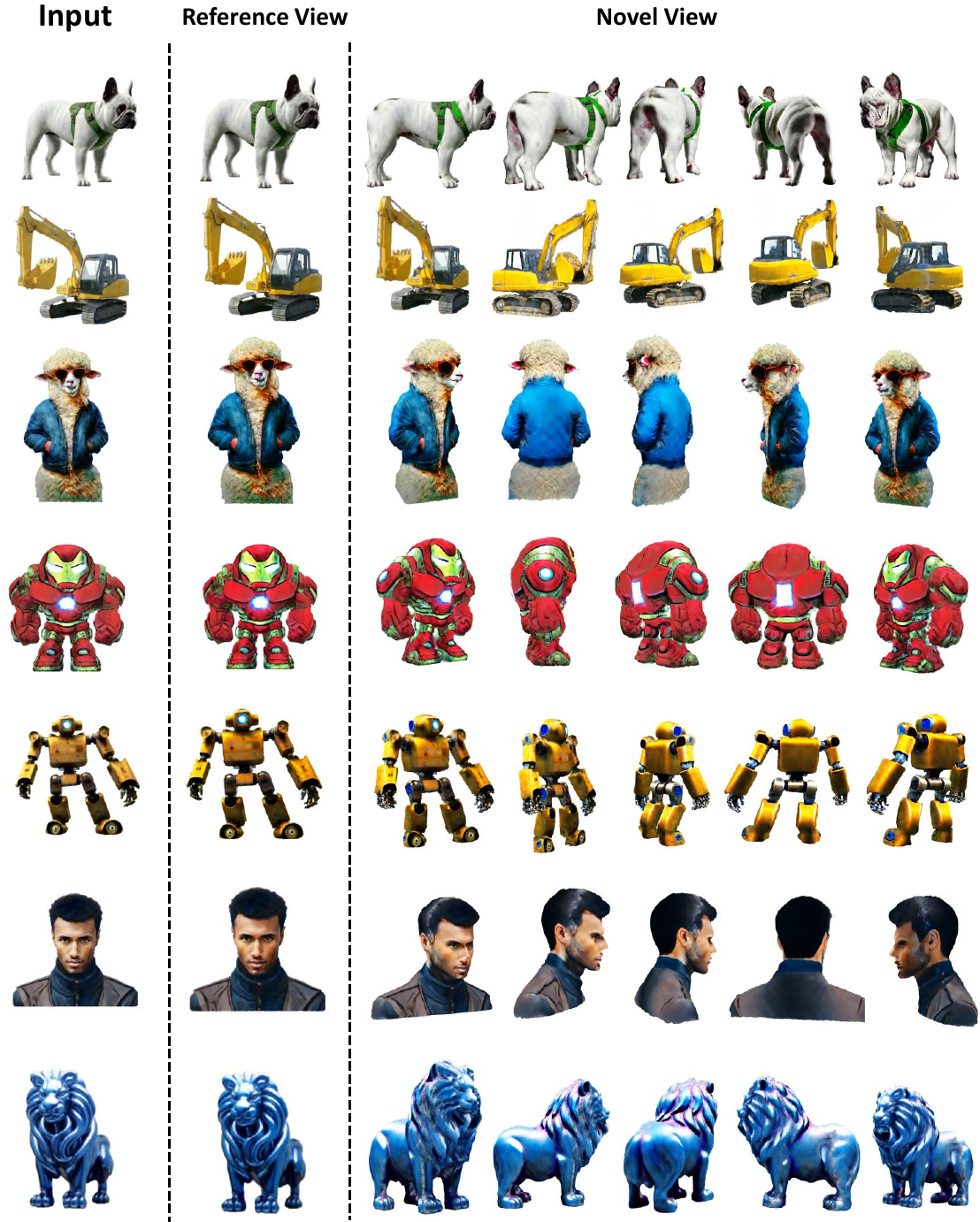}
    \caption{More DTC123-generated results. Our approach yields results with improved fidelity and more robust geometry.}
    \label{fig:big}
\end{figure*}

\begin{figure*}[t]
    \centering
    \includegraphics[width=0.95\linewidth]{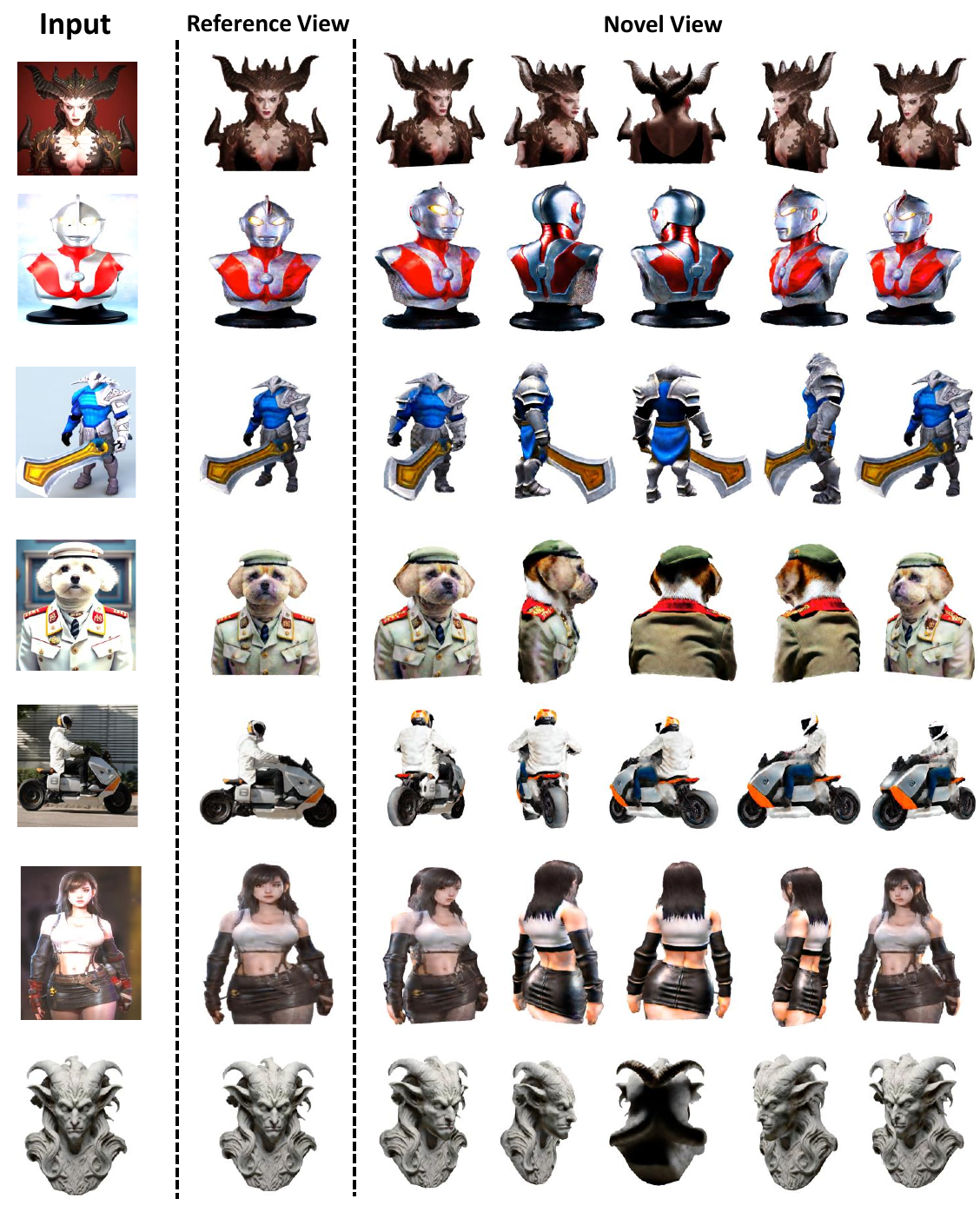}
    \caption{More DTC123-generated results. Our approach yields results with enhanced fidelity and more robust geometry.}
    \label{fig:big2}
\end{figure*}